\documentclass[reqno,12pt]{amsart}
\usepackage{fullpage}
\usepackage{url}
\usepackage{bbm}
\usepackage{amsmath,amssymb,amsthm}
\usepackage{mathtools}
\usepackage{graphicx}
\usepackage{algorithm}
\usepackage{algpseudocode}
\usepackage{enumerate}
\usepackage[colorlinks=true]{hyperref}
\usepackage{appendix}

\DeclareMathOperator{\sgn}{sgn}
\DeclareMathOperator{\polylog}{polylog}

\newtheorem{theorem}{Theorem}[section]

\newtheorem{lemma}[theorem]{Lemma}
\newtheorem{proposition}{Proposition}

\theoremstyle{definition}
\newtheorem{definition}[theorem]{Definition}
\newtheorem{remark}{Remark}

\newcommand{\thr}{\theta}

\newcommand{\E}{\mathbb{E}}

\newcommand{\R}{\mathbb{R}}
\renewcommand{\P}{\mathbb{P}}

\newcommand{\N}{\mathbb{N}}

\begin{document}


\title{Unified Stochastic Framework for Neural Network Quantization and Pruning}
\author[H. Zhang]{Haoyu Zhang\textsuperscript{1}\textsuperscript{$\dagger$}}
\thanks{\textsuperscript{1} Department of Mathematics, University of California-San Diego, La Jolla, California 92093, USA.
(haz053@ucsd.edu)}
\address{(HZ) Department of Mathematics, University of California-San Diego, La Jolla, California 92093, USA.} \email{haz053@ucsd.edu}
\author[R. Saab]{Rayan Saab\textsuperscript{2}}
\thanks{\textsuperscript{2} Department of Mathematics, Hal{\i}c{\i}o{\u g}lu Data Science Institute, University of California-San Diego, La Jolla, California 92093, USA.
(rsaab@ucsd.edu)}
\thanks{\textsuperscript{$\dagger$} Corresponding author}
\address{(RS) Department of Mathematics, Hal{\i}c{\i}o{\u g}lu Data Science Institute, University of California-San Diego, La Jolla, California 92093, USA.} \email{rsaab@ucsd.edu}
\maketitle

\begin{abstract}

Quantization and pruning are two essential techniques for compressing neural networks, yet they are often treated independently, with limited theoretical analysis connecting them. This paper introduces a unified framework for post-training quantization and pruning using stochastic path-following algorithms. Our approach builds on the Stochastic Path Following Quantization (SPFQ) method, extending its applicability to pruning and low-bit quantization, including challenging 1-bit regimes. By incorporating a scaling parameter and generalizing the stochastic operator, the proposed method achieves robust error correction and yields rigorous theoretical error bounds for both quantization and pruning as well as their combination. 


\end{abstract}



\section{Introduction}

Modern deep neural networks (DNNs) have achieved significant success but require substantial memory and computation due to their large number of parameters. Model compression techniques, including quantization, pruning, knowledge distillation, and low-rank decomposition, mitigate these issues. We focus on quantization and pruning.


Quantization replaces the weights in a DNN with elements from a finite set, allowing them to be represented with fewer bits and consequently not only compresses the network but also accelerates inference. DNN quantization methods include quantization-aware training and post-training quantization. Quantization-aware training (e.g., \cite{CYDGMK, CWVCPSG, CBD, JKCZTHAK, WLLLH, ZYYH, ZYGXC}) modifies standard neural network training to account for weight quantization, often requiring retraining and hyperparameter tuning, which can be computationally expensive, especially if one needs to repeat the process for multiple bit depths. In contrast, post-training quantization (e.g., \cite{CKYK, CGFZS, FAHA, HNHBS, LS, MS, NAVLB, WCHC, ZZS, ZHDDZ}) operates on pre-trained models, using relatively little data to replace floating-point parameters with quantized counterparts. Post-training quantization methods are generally much less demanding than the original training process, allowing for easy re-quantization of a base model at various bit depths.


Meanwhile, pruning techniques aim to set as many of a DNN's weights to zero as possible. Pruning techniques can also be categorized into those that operate pre-training,  during training, or post-training. Pre-training pruning methods (e.g., \cite{LAT, SCCWGWL, TKYG, WZG}) prune the model before training and often rely significantly on good initialization. Methods that prune during training  (e.g., \cite{EGMCE, HW, HKDFY, LMZGYCS, ZNZZZT}) attempt to identify and remove unnecessary parameters during training, and can increase overall training cost. Post-training pruning methods (e.g., \cite{ACNHH, FA, LZKZXWCYLZ, MFW}) prune already trained models and fine-tune them if necessary. Similar to their quantization counterparts, post-training methods allow models to be pruned to various sparsity levels, enabling implementation on platforms with different memory, power, and computational resources. 

In this theory-focused paper, we extend post-training algorithms for quantization and pruning by building on the work presented in \cite{ZS}, which introduced Stochastic Path Following Quantization (SPFQ). SPFQ is a stochastic error-correcting algorithm that sequentially quantizes a neuron's weights to minimize accumulated error, with an accompanying theoretical upper bound on the Euclidean error that scales as $\sqrt{\min\{m,N\}}\cdot \text{polylog}(N)$, where $N$ is the dimension of the neuron and $m$ is the number of calibration data points. This scaling compares favorably to the $N$ scaling associated with the most common Round-to-Nearest (RTN) method. However, SPFQ's focus on minimizing accumulated error without constraints limits its applicability to extremely low-bit quantizers, as it cannot guarantee rigorous error bounds in such cases. Additionally, while SPFQ includes an analysis for finite-bit quantization when the number of bits is not too small, its framework does not readily extend to pruning.

We provide some important generalizations to SPFQ and its analysis. We show that one can replace the quantization operator in SPFQ by any stochastic operator that satisfies certain properties, and we introduce a scaling parameter into the argument of this operator. Loosely speaking, this scaling parameter controls the amount of error correction in the algorithm, and is critical to allowing us to handle the low-bit (even 1-bit) quantization regime. These modifications make our approach versatile for various tasks, including pruning and the afore-mentioned quantization, and we prove rigorous error bounds for these applications.

This paper is organized as follows. Section \ref{notation} introduces  notation. In Section \ref{algorithm}, we present our algorithm and the intuition behind it. Section \ref{theory} states and proves the main theoretical result. Sections \ref{onebit}, \ref{sparse}, and \ref{quant_prune} discuss applications —quantization, pruning, and combined quantization with pruning—and apply the main theorem to derive error bounds, demonstrating the algorithm’s versatility for low-bit quantization and sparsification.

\section{Notation}\label{notation}
We define an $ L $-layer perceptron $ \Phi : \mathbb{R}^{N_{0}} \rightarrow \mathbb{R}^{N_{L}} $ via its action
$$
\Phi(x) = \rho^{(L)} \circ A^{(L)} \circ \rho^{(L-1)} \circ A^{(L-1)} \circ \dots \circ \rho^{(1)} \circ A^{(1)}(x),
$$
where $ \rho^{(i)} : \mathbb{R}^{N_{i}} \rightarrow \mathbb{R}^{N_{i}} $ is an entry-wise nonlinear activation function, and $A^{(i)} : \mathbb{R}^{N_{i - 1}} \rightarrow \mathbb{R}^{N_{i}} $ is an affine map defined by $ A^{(i)}(x) = W^{(i)\top} x + b^{(i)} $. Here, $ W^{(i)} \in \mathbb{R}^{N_{i-1} \times N_{i}} $ is the \emph{weight} matrix and $ b^{(i)} \in \mathbb{R}^{N_{i}} $ is the \emph{bias} vector in the $ i $-th layer. It is worth noting that $ W^{(i)\top} x + b^{(i)} $ can be rewritten as $ (W^{(i)\top} \; b^{(i)}) (x^{\top} \; 1)^{\top} $. Therefore, we can absorb the bias vector into the weight matrix by extending the data $ x $ to $ (x^{\top} \; 1)^{\top} $. As a result, we  assume throughout this paper without loss of generality, that the bias vector $ b^{(i)} = 0 $. Additionally, we assume $ \max_{i,j,k} |W^{(i)}_{jk}| \leq K $. 

We denote by $\Phi^{(i)}$ the composition of the first $i$ layers of the model, so that
$$
\Phi^{(i)}(x) = \rho^{(i)}\circ A^{(i)}\circ \rho^{(i-1)}\circ A^{(i-1)}\circ...\circ \rho^{(1)}\circ A^{(1)}(x)
$$
represents the \emph{activations} at layer $i$. Similarly, we denote the compressed model by $\widetilde{\Phi}$, and its first $i$ layers by $\widetilde{\Phi}^{(i)}$. $X\in\mathbb{R}^{m\times N_{0}}$ represents the data used for compressing the network, where each row is a data point in $\mathbb{R}^{N_0}$. We extend  notation so $\Phi^{(i)}(X) \in \mathbb{R}^{m\times N_{i}}$ denotes the application of $\Phi^{(i)}$ to each row of $X$. 

For a matrix $M\in\mathbb{R}^{m\times n}$, $M_{j}\in\mathbb{R}^{m}$ represents the $j$-th column of $M$, $M_{ij}$ represents the $(i,j)$-th entry of $M$, and for a vector $v\in\mathbb{R}^{n}$, $v_{j}\in\mathbb{R}$ represents the $j$-th entry of $v$. If $A$ and $B$ are two symmetric matrices, $A\preceq B$ means $B-A$ is positive semi-definite. For vectors $x\in \R^d$, we require the $\ell_2$ and $\ell_\infty$ norms $\|x\|_{\infty}:=\max_i |x_i|$, $\|x\| := \sqrt{\sum_i x_i^2}$. Meanwhile for matrices $X$, $\|X\|_{F} = \sqrt{\sum_{i,j}X_{ij}^2}$ denotes its Frobenius norm. We use $\mathcal{N}(\mu , \Sigma)$ to denote a Gaussian random vector with mean $\mu$ and covariance matrix $\Sigma$, and $U[a,b]$ to denote a random variable drawn uniformly on the interval $[a,b]$. Finally, 
$P_{v}$, where $v$ is a vector, denotes the orthogonal projection onto $\text{span}(v)$.

\section{Stochastic Path Following Compression}\label{algorithm}

We focus on the $ L $-layer perceptron\footnotemark\footnotetext{Our algorithm and theory are also applicable to convolutional neural networks, but for simplicity we consider only multi-layer perceptrons (see, e.g.,\cite{ZZS})}
 $ \Phi $ defined above. The aim is to compress the weights of $ \Phi $ through quantization or pruning to obtain a network $\widetilde{\Phi}$ with $\Phi(X) \approx \widetilde{\Phi}(X)$. We accomplish this via Algorithm \ref{algo} below which we now briefly explain.

 \begin{algorithm}[h]
    \caption{\textsc{Stochastic Path Following Compression}($\mathcal{T},\thr,C,\Phi,X$)}
    \label{algo}
    \begin{algorithmic}        
    \State\textbf{Parameters:} Stochastic operator $ \mathcal{T} $, constant $ C \geq 1 $.
    \State\textbf{Input:} $ L $-layer perceptron $ \Phi $ with weight matrices $ W^{(i)} \in \mathbb{R}^{N_{i-1} \times N_{i}} $, data $ X \in \mathbb{R}^{m \times N_{0}}$. 
    \State\textbf{Output:} Compressed neural network $ \widetilde{\Phi} $.
    \end{algorithmic}
    \bigskip
    \begin{algorithmic}[1]
        \For{$ i = 1 $ to $ L $}
            \State Set $ X^{(i-1)} = \Phi^{(i-1)}(X) \in \mathbb{R}^{m \times N_{i-1}} $.
            \State Set $ \widetilde{X}^{(i-1)} = \widetilde{\Phi}^{(i-1)}(X) \in \mathbb{R}^{m \times N_{i-1}} $.
            
            \Repeat
                \State Choose a column $ w \in \mathbb{R}^{N_{i-1}} $ of $ W^{(i)} $.
                \State Run Algorithm \ref{phaseii} to obtain ``compressed" weights $ q $: $$ q = \textsc{Compression}(\mathcal{T}, C, w, X, \widetilde{X}). $$
            \Until{All columns of $ W^{(i)} $ are 
            compressed into $Q^{(i)}$.}
            
            \State Set: $\widetilde{\Phi}^{(i)}(\cdot) = \rho^{(i)}\circ \widetilde{A}^{(i)}\circ \rho^{(i-1)}\circ...\circ \rho^{(1)}\circ \widetilde{A}^{(1)}(\cdot)$, where $\widetilde{A}^{(i)}(x)=Q^{(i)\top}x.$

        \EndFor
    \end{algorithmic}
\end{algorithm}

 Algorithm \ref{algo} uses $m$ data points in $\R^{N_0}$ represented by $ X\in\mathbb{R}^{m\times N_{0}} $ to guide the compression. The basic idea is to perform the compression one layer at a time, starting with the first. Thus, at layer $i$, we have access to 
 $X^{(i-1)} := \Phi^{(i-1)}(X)$ and $ \widetilde{X}^{(i-1)}:=\widetilde\Phi^{(i-1)}(X)$, both in $\mathbb{R}^{m\times N_{i-1}}$, and the goal is to compress $ W^{(i)}\in\mathbb{R}^{N_{i-1}\times N_{i}}$ into $Q^{(i)}\in\mathbb{R}^{N_{i-1}\times N_{i}}$ such that $X^{(i-1)}W^{(i)} \approx \widetilde{X}^{(i-1)}Q^{(i)}$. To accomplish this, we use Algorithm \ref{phaseii} to independently, and possibly in parallel, replace  each column (neuron) $w 
 \in\R^{N_{i-1}} $ of $ W^{(i)} $,  entry by entry, with a new neuron $ q$ such that $X^{(i-1)}w\approx\widetilde{X}^{(i-1)}q$. 
 In Algorithm \ref{phaseii}, to ensure $X^{(i-1)}w\approx\widetilde{X}^{(i-1)}q$, we proceed sequentially by computing, for $ t=1, 2,...,N_{i-1} $, the quantized or pruned coefficient $q_t$ that aligns the sums $\sum_{j=1}^{t}w_{j}X^{(i-1)}_{j} $ and $ \sum_{j=1}^{t}q_{j}\widetilde{X}^{(i-1)}_{j} $. This step forms the core of our method. It balances two competing requirements, namely reducing the error while ensuring that \(q_t\) lies within the quantization alphabet or has a high likelihood of being set to zero, depending on the application.

Denote the accumulated error at step $t-1$ by $ u_{t-1}:= \sum_{j=1}^{t-1}w_{j}X^{(i-1)}_{j} - \sum_{j=1}^{t-1}q_{j}\widetilde{X}^{(i-1)}_{j} $. To 
balance these requirements, we choose $ q_{t} $ such that $  \tfrac{1}{C}\cdot u_{t-1}+w_{t}X^{(i-1)}_{t}-q_{t}\widetilde{X}^{(i-1)}_{t}\approx0$, where $ C\geq1$ controls the trade-off between correcting the prior error and only approximating $w_t$. If we were picking $q_t$ to minimize $\| \tfrac{1}{C}\cdot u_{t-1}+w_{t}X^{(i-1)}_{t}-q_{t}\widetilde{X}^{(i-1)}_{t}\|$ over $\mathbb{R}$, in the absence of quantization or sparsity constraints, $ q_{t} $ is best chosen by projecting $ \tfrac{1}{C}\cdot u_{t-1}+w_{t}X^{(i-1)}_{t} $ onto $ \widetilde{X}^{(i-1)}_{t} $, yielding $ q_{t}=\tfrac{\langle h_{t},\widetilde{X}^{(i-1)}_{t}\rangle}{C\|\widetilde{X}^{(i-1)}_{t}\|^{2}} $, where  $ h_{t}=Cw_{t}X^{(i-1)}_{t}+u_{t-1}  $. However, since our goal is quantization or pruning, we instead set $ q_{t} = \mathcal{T}\left(\tfrac{\langle h_{t},\widetilde{X}^{(i-1)}_{t}\rangle}{C\|\widetilde{X}^{(i-1)}_{t}\|^{2}}\right) $, where $ \mathcal{T}$ is a stochastic operator that promotes sparsity and/or quantizes its argument. 
Each neuron $ w $ can be handled this way in parallel to compress the entire layer.

In the case of quantization, if the ratio $ \tfrac{\langle h_{t},\widetilde{X}^{(i-1)}{t}\rangle}{C|\widetilde{X}^{(i-1)}{t}|^{2}} $ falls outside the range of the chosen alphabet, the resulting error can become unbounded. In \cite{ZS}, where $C = 1$, this issue is avoided by enlarging the alphabet to include more quantization levels. In contrast, to work with a fixed 1-bit alphabet, we allow a scaling constant $C \geq 1$ to ensure the quantization remains controlled. In Theorem~\ref{prop:onebit_quant}, we show that choosing $C \sim \log(N_0 N_1)$, where $N_0$ and $N_1$ are the input and output channel dimensions respectively, guarantees that the quantized weights remain within the 1-bit alphabet, while still achieving a logarithmic reconstruction error bound on $\|XW - XQ\|_\infty$ with high probability..

In the remainder of this paper, we show that these modifications allow us to control the error with high probability on the stochasticity of the operator $\mathcal{T}$.

\begin{algorithm}
    \caption{\textsc{Compression}($\mathcal{T}, C, w, X, \widetilde{X}$)}
    \label{phaseii}
        \begin{algorithmic}
        \State \textbf{Parameters:} Stochastic operator $ \mathcal{T} $, constant $ C \geq 1 $.
        \State \textbf{Input:} Neuron $ w = (w_{1}, w_{2}, \dots, w_{N})^{\top} \in \mathbb{R}^{N} $, data (activations) $ X = (X_{1}, X_{2}, \dots, X_{N}) \in \mathbb{R}^{m \times N}$, activations from quantized network $ \widetilde{X} = (\widetilde{X}_{1}, \widetilde{X}_{2}, \dots, \widetilde{X}_{N}) \in \mathbb{R}^{m \times N}$. 
        \State \textbf{Output:} Compressed weights $ q = (q_1, \dots, q_{N})^{\top}$.

        \end{algorithmic}
    \bigskip
    
    \begin{algorithmic}[1]
        \State Initialize accumulated error $ u_{0} = 0$.
        \For{$ t = 1 $ to $ N $}
            \State Compute $ h_{t} = C w_{t} X_{t} + u_{t-1} $
            \State Compute $ v_{t} = \tfrac{\langle h_{t}, \widetilde{X}_{t} \rangle}{C \| \widetilde{X}_{t} \|^{2}} $
            \State Compute $ q_{t} = \mathcal{T}(v_t) $
            \State Update accumulated error: $ u_{t} = u_{t-1} + w_{t} X_{t} - q_{t} \widetilde{X}_{t} $
        \EndFor
    \end{algorithmic}
\end{algorithm}

\section{Theoretical analysis}\label{theory}

The primary focus of our theoretical analysis of Algorithm 1 is the error analysis for a single layer of the network. Specifically, we aim to bound the error introduced when compressing the weight matrix of a single layer.
 Consider the first layer of a neural network, characterized by the weight matrix \( W = (W_1, W_2, \ldots, W_{N_1}) \in \mathbb{R}^{N_0 \times N_1} \), where each \( W_i \) represents a neuron (column) of \( W \)\footnote{One could also consider any arbitrary layer and work with $X$ being either the corresponding activations of the compressed or original network}. Given the input data \( X \in \mathbb{R}^{m \times N_0} \), Algorithm 1 compresses \( W \) into \( Q = (Q_1, Q_2, \ldots, Q_{N_1}) \) and the error of this compression is measured in the Frobenius norm,
\[
\| XW - XQ \|_F^2 = \sum_{i=1}^{N_1} \| X(W_i - Q_i) \|_2^2.
\]
Since this function is separable with respect to \( Q_1, \ldots, Q_{N_1} \), Algorithm 1 processes each column \( W_1, W_2, \ldots, W_{N_1} \) independently to produce \( Q_1, Q_2, \ldots, Q_{N_1} \).
To simplify notation, we represent a single column by lowercase letters \( w \) and \( q \), denoting the uncompressed and compressed versions, respectively. The error analysis for the entire layer thus reduces to bounding \( \| X(w - q) \|_2^2 \), which corresponds to the error from Algorithm 2. The remainder of this section is dedicated to proving Theorem 4.2, which provides explicit bounds on the accumulated error at each step \( t \). Its proof requires the concept of convex ordering which we will first introduce. Then, we state and prove Theorem \ref{lemma:ubound} using a series of technical lemmas, which we also subsequently prove.

\subsection{Main Theorem}\label{tildeu}
We now define \textit{convex ordering} \cite{SS}, a notion that will allow us to control random vectors generated by our algorithm by replacing them with Gaussian random vectors.
\begin{definition}
	Let $ X, Y $ be $ n$-dimensional random vectors such that
	\begin{gather*}
		\mathbb{E}f(X)\leq\mathbb{E}f(Y)
	\end{gather*}
holds for all convex functions $ f:\mathbb{R}^{n}\rightarrow \mathbb{R} $, provided the expectations exist. Then $ X $ is said to be less than $ Y $ in the \textit{convex order}, denoted by $ X\prec_{cx}Y $.
\end{definition}
Recall that we initialize the accumulated error as $ u_{0}=0 $ and note that at each step $ t $, we compute $ h_{t}= Cw_{t}X_{t}+u_{t-1}$. Then we use $ h_{t} $ to compute $ v_{t} = \tfrac{\langle h_{t},X_{t}\rangle}{C\|X_{t}\|^{2}}$ and compute the compressed entry $ q_{t}=\mathcal{T}(v_{t}) $. Finally, we update the accumulated error $ u_{t}=u_{t-1}+w_{t}X_{t}-q_{t}X_{t} $ and move on to the next step. 

Using convex ordering, we can prove that $ u_i $ is dominated by a Gaussian random vector. This result, stated in the theorem below, allows us to control the errors and probabilities for the quantization and pruning applications in Sections \ref{onebit}, \ref{sparse}, and \ref{quant_prune}. Here we adapt the proof technique from \cite{ALS}.
\begin{theorem}\label{lemma:ubound}
	Let $C\geq 1$. Assume the weight vector satisfies $ \|w\|_{\infty} < K $. Further, assume that the operator $ \mathcal{T} $ satisfies $\mathbb{E}[\mathcal{T}(v)] = v$ and $ |\mathcal{T}(v) - v| \leq M $
for all $ v \in\mathbb{R}$. Then, by running Algorithm \ref{phaseii} with $w\in\mathbb{R}^{N_0}$  and $X\in\mathbb{R}^{m\times N_0}$, the resulting $ u_t $ satisfies
$$
u_{t} \prec_{cx} \mathcal{N}(0, \beta_{t}I), \quad \text{for } 1 \leq t \leq N,
$$
where $ \beta_{t} $ is given by
$$
\beta_{t} = \frac{C \pi M^{2}}{2} \max_{1 \leq i \leq t} \|X_{i}\|^{2}.
$$
\end{theorem}
\begin{proof}
First by Lemma \ref{lemma:ubound1}, we deduce $u_t\prec_{cx} \mathcal{N}(0,\Sigma_t)$, where $\Sigma_t$ is defined as in Lemma \ref{lemma:ubound1}. Further by Lemma \ref{lemma:matrixbound}, we have $\Sigma_t\preceq \beta_{t}I$. We then apply Lemmas \ref{lemma:convex_order}, \ref{lemma:normal_dom} and \ref{trans} to complete the proof.
\end{proof}
\subsection{Technical Lemmas}
In this subsection, we present the lemmas used for the proof of Theorem \ref{lemma:ubound}. We start with a lemma collecting various  results on convex ordering,  with the proofs in \cite{ALS, SS}.

\begin{lemma}\label{lemma:convex_order}The following all hold:
	\begin{enumerate}[a.]
		\item \label{trans} Let $X$, $Y$ and $Z$ be $n$-dimensional random vectors. If $ X\prec_{cx}Y $ and $ Y\prec_{cx}Z $, then $ X\prec_{cx}Z $. (Lemma 2.3 in \cite{ALS})
		\item \label{linear}  Let $X$, $Y$ be $n$-dimensional random vectors. If $ X\prec_{cx}Y $, then for any matrix $ M \in \mathbb{R}^{n\times n} $, we have $ MX\prec_{cx}MY $. (Lemma A.3 in \cite{ZS})
            \item \label{lemma:normal_dom} 	If $ A$ and $B$ are two $n\times n$ positive semi-definite matrices and $ A \preceq B $, then $ \mathcal{N}(0, A)\prec_{cx} \mathcal{N}(0, B) $. (Lemma A.2 in \cite{ZS})
		\item \label{lemma:cx-sum} Let $X$, $Y$, $W$, and $Z$ be $n$-dimensional random vectors and suppose that $X$ and $Y$ live on the same probability space. Let $W$ and $Z$ be independent and suppose that $X\prec_\mathrm{cx} W$ and $(Y-X)|X \prec_\mathrm{cx} Z$. Then $Y\prec_\mathrm{cx} W+Z$. (Lemma 2.5 in \cite{ALS})
		\item \label{lemma:cx-bounded} 	Let $X$ be a real-valued random variable with $\E X = 0$ and $|X|\leq C$. Then $X\prec_\mathrm{cx}\mathcal{N}\bigl(0, \frac{\pi C^2}{2}\bigr)$. (Lemma 2.6 in \cite{ALS})
		\item \label{lemma:cx-gaussian-tail} 	Let $X$ be an $n$-dimensional random vector such that $X\prec_\mathrm{cx}\mathcal{N}(\mu, \sigma^2 I)$, and let $\alpha>0$. Then 
		$$
    		\P\bigl(\|X-\mu\|_\infty \leq \alpha \bigr) \geq 1- \sqrt{2} n e^{-\frac{\alpha^2}{4\sigma^2}}.
		$$
		In particular, if $\alpha=2\sigma\sqrt{\log(\tfrac{\sqrt{2}n}{\gamma})}$ with $\gamma\in(0,1]$, we have 
		\begin{gather*}
		\P\big(\|X-\mu\|_\infty \leq 2\sigma\sqrt{\log(\tfrac{\sqrt{2}n}{\gamma})}\big)\geq 1-\gamma.
		\end{gather*} (Lemma B.2 in \cite{ZS})
	\end{enumerate}
\end{lemma}
Using Lemma \ref{lemma:convex_order} allows to deduce the following result on the distribution of $u_j$. We use a proof technique similar to that in \cite{ALS}.
\begin{lemma}\label{lemma:ubound1}
	Under the same assumptions as in Theorem \ref{lemma:ubound}, by running Algorithm \ref{phaseii} with $w\in\mathbb{R}^{N_0}$  and $X\in\mathbb{R}^{m\times N_0}$, we obtain $u_t $ satisfying
	\begin{gather*}
		u_t\prec_{cx} \mathcal{N}(0,\Sigma_{t}) \text{ for $ 1\leq t \leq N $},
	\end{gather*}
	where $ \Sigma_{0}=0$ and $ \Sigma_{t}, t\geq 1 $ is defined recursively as 
	\begin{gather*}
		\Sigma_{t}=\Big(I-\dfrac{P_{X_{t}}}{C}\Big)\Sigma_{t-1}\Big(I-\dfrac{P_{X_{t}}}{C}\Big)+\dfrac{\pi M^{2}}{2}X_{t}X_{t}^{\top}.
	\end{gather*}
\end{lemma}
\begin{proof}
	We use induction. For the base case, we have $u_0=0$, so the statement is trivially true. Now, assume $u_{t-1} \prec_{cx} \mathcal{N}(0,\Sigma_{t-1})$. 
    By the update formula of $q_t$ in Algorithm \ref{phaseii}:
 \begin{align*}
 u_{t}& =u_{t-1}+w_t X_t - q_t X_t\\
 &=u_{t-1}+w_t X_t - v_t X_t + (v_t - q_t)X_t\\
 &=\Big(I-\dfrac{P_{X_{t}}}{C}\Big)u_{t-1} + (v_{t}-\mathcal{T}(v_{t}))X_t.
 \end{align*}
 The third equality is due to the definition of $v_t$ and $q_t$. 
By the assumption on $\mathcal{T}$, and conditioning on $u_{t-1}$, $(v_{t}-\mathcal{T}(v_{t})\mid u_{t-1}) $ is mean zero and $|v_{t}-\mathcal{T}(v_{t})|\leq M  $. We then use \ref{lemma:cx-bounded} from Lemma \ref{lemma:convex_order},  to deduce $ (v_{t}-\mathcal{T}(v_{t})\mid u_{t-1}) \prec_{cx}\mathcal{N}(0,\tfrac{\pi M^{2}}{2})$. Further applying Lemma \ref{lemma:convex_order}, \ref{linear}, we obtain 
	\begin{gather*}
		\left((v_{t}-\mathcal{T}(v_{t}))X_{t}\mid u_{t-1}\right)\prec_{cx}\mathcal{N}\Big(0,\dfrac{\pi M^{2}}{2}X_{t}X_{t}^{\top}\Big).
	\end{gather*}
	Since $ u_{t-1}\prec_{cx}\mathcal{N}(0,\Sigma_{t-1}) $ by the induction hypothesis, we similarly have
	\begin{gather*}
		\Big(I-\dfrac{P_{X_{t}}}{C}\Big)u_{t-1}\prec_{cx}\mathcal{N}\Big(0,\Big(I-\dfrac{P_{X_{t}}}{C}\Big)\Sigma_{t-1}\Big(I-\dfrac{P_{X_{t}}}{C}\Big)\Big).
	\end{gather*}
	Finally, by Lemma \ref{lemma:convex_order}, \ref{lemma:cx-sum} with 
    $\left(I-\frac{P_{X_{t}}}{C}\right)u_{t-1}$ in place of $X$, $\left(v_{t}-\mathcal{T}(v_{t})\right)X_t \mid u_{t-1}$ in place of $Y-X\mid X$  
    and $\mathcal{N}\left(0,\left(I-\frac{P_{X_{t}}}{C}\right)\Sigma_{t-1}\left(I-\tfrac{P_{X_{t}}}{C}\right)\right)$ in place of $W$, with  $\mathcal{N}\left(0,\tfrac{\pi M^{2}}{2}X_{t}X_{t}^{\top}\right)$ as $Z$,
    where $W$ and $Z$ are chosen independently, we conclude
	\begin{gather*}
		u_{t}\prec_{cx}\mathcal{N}\Big(0,\Big(I-\dfrac{P_{X_{t}}}{C}\Big)\Sigma_{t-1}\Big(I-\dfrac{P_{X_{t}}}{C}\Big)+\dfrac{\pi M^{2}}{2}X_{t}X_{t}^{\top}\Big).
	\end{gather*}
	This completes the induction.
\end{proof}

Having controlled our object of interest by a Gaussian random vector with a known covariance matrix,  $ \Sigma_{t} $ defined in Lemma \ref{lemma:ubound1}, we next control the covariance matrix itself. 
\begin{lemma}\label{lemma:matrixbound}
	Let $C\geq1$, then for the above defined $ \Sigma_{t} $, we have
	\begin{gather*}
		\Sigma_{t}\preceq\beta_{t}I,
	\end{gather*}
	where $\beta_0 =0$ and $\beta_{t}=\dfrac{C\pi M^{2}}{2}\max_{1\leq i\leq t}\|X_{i}\|^{2},$ when $t \geq 1 $.
\end{lemma}
\begin{proof}
	The base case is obvious. Let us now proceed by induction. Suppose $\Sigma_{t-1} \preceq \beta_{t-1} I$. Then 
	\begin{align*}
		\Sigma_{t}&=\Big(I-\dfrac{P_{X_{t}}}{C}\Big)\Sigma_{t-1}\Big(I-\dfrac{P_{X_{t}}}{C}\Big)+\dfrac{\pi M^{2}X_{t}X_{t}^{\top}}{2}\\
		&\preceq \beta_{t-1}\Big(I-\dfrac{P_{X_{t}}}{C}\Big)^{2}+\dfrac{\pi M^{2}X_{t}X_{t}^{\top}}{2}\\
		&=\beta_{t-1}\Big(I+\dfrac{1}{C^{2}}\dfrac{X_{t}X_{t}^{\top}}{\|X_{t}\|^{2}}-2\dfrac{X_{t}X_{t}^{\top}}{C\|X_{t}\|^{2}}\Big)+\dfrac{\pi M^{2}X_{t}X_{t}^{\top}}{2}\\
		&\preceq\beta_{t}\Big(I+\dfrac{1}{C^{2}}\dfrac{X_{t}X_{t}^{\top}}{\|X_{t}\|^{2}}-2\dfrac{X_{t}X_{t}^{\top}}{C\|X_{t}\|^{2}}\Big)+\dfrac{\pi M^{2}X_{t}X_{t}^{\top}}{2}\\
		&=\beta_{t}I+\dfrac{1}{\|X_{t}\|^{2}}\Big(\beta_{t}\Big(\dfrac{1}{C^{2}}-\dfrac{2}{C}\Big)+\dfrac{\pi M^{2}}{2}\|X_{t}\|^{2}\Big)X_{t}X_{t}^{\top}\\
		&=\beta_{t}I+\dfrac{\pi M^2}{\|X_{t}\|^{2}}\Big(\dfrac{C}{2}\max\limits_{1\leq i\leq t}\|X_{i}\|^{2} \Big(\dfrac{1-2C}{C^{2}}\Big)+\dfrac{1}{2}\|X_{t}\|^{2}\Big)X_{t}X_{t}^{\top}.
  \end{align*}
  Since $C\geq1$, we know $1-2C \leq 0$, then we can further deduce
  \begin{align*}
		\Sigma_{t}&\preceq \beta_{t}I+\dfrac{\pi M^{2}}{\|X_{t}\|^{2}}\Big(\dfrac{C}{2}\|X_{t}\|^{2} \Big(\dfrac{1-2C}{C^{2}}\Big)+\dfrac{1}{2}\|X_{t}\|^{2}\Big)X_{t}X_{t}^{\top}\\
		&=\beta_{t}I+\Big(\dfrac{\pi M^{2}}{2} \Big(\dfrac{1}{C}-2\Big)+\dfrac{\pi M^{2}}{2}\Big)X_{t}X_{t}^{\top}\\
		&\preceq \beta_{t}I+\Big(-\dfrac{\pi M^{2}}{2} +\dfrac{\pi M^{2}}{2}\Big)X_{t}X_{t}^{\top}\\
		&= \beta_{t}I.
	\end{align*}
	This completes the induction.
\end{proof}
\section{One-bit Quantization}\label{onebit}

We are now ready to first describe the selection of \( \mathcal{T} \) for one-bit quantization and then provide the corresponding error analysis for running Algorithm \ref{algo} on a single layer of a neural network. Notably, our analysis extends with minimal additional effort to finite multi-bit quantization, a consideration we leave to the interested reader.

\subsection{Choice of $ \mathcal{T} $ for one-bit quantization}

Suppose we want to quantize an $ L $-layer perceptron which has all its weights strictly bounded in absolute value by $ K>0 $, and our goal is to achieve one-bit quantization. We begin by considering an infinite alphabet $\mathcal{A}= \{\dots,-10K,-6K,-2K,2K,6K,10K,\dots\} $. Then, we choose the operator $\mathcal{T}$ to be the stochastic scalar quantizer $ \mathcal{Q}:\mathbb{R}\rightarrow \mathcal{A}$, with \begin{equation*}
		\mathcal{Q}(z)=
		\begin{cases}
                    \left(\lfloor \frac{z}{4K}\rfloor\right)\cdot 4K & \text{with probability $ p_z $}\\
				 \left(\lceil  \frac{z}{4K} \rceil\right)\cdot 4K & \text{with probability $ 1-p_z $},\\
			\end{cases}
	\end{equation*}
where $ p_z=\lceil  \frac{z}{2K} \rceil - \frac{z}{2K} $. 

 It is straightforward to verify that the operator $\mathcal{Q}$ satisfies the assumptions in Theorem \ref{lemma:ubound} with $M=4K$. We will show that with this choice of operator, with high probability, Algorithm \ref{algo} gives logarithmic reconstruction error with respect to the input dimension, and, crucially, only ever uses two values from the alphabet: $-2K$ and $2K$. That is, although the quantizer is defined over an infinite alphabet, the output of the algorithm effectively requires only 1-bit quantization with high probability, while maintaining small errors.
\subsection{Error analysis}
Consider one layer of a neural-network, and denote its action by $ \Phi(x)=\rho(W^{\top}x) $, where $ x\in\mathbb{R}^{N_{0}} $ represents the input data or activations, $ W\in\mathbb{R}^{N_{0}\times N_{1}} $ is the weight matrix and $ \rho $ is the ReLU activation function, or any (say) 1-Lipschitz function. Now, we use Algorithm \ref{algo} to quantize $ \Phi $ using $ X \in\mathbb{R}^{m\times N_{0}}$, where each row represents a data point. The following holds.

\begin{proposition}\label{prop:onebit_quant}  
	 Let $C\geq 1$ and fix any $p\geq1$. Suppose that $Q$ is the quantized weight matrix resulting from Algorithm \ref{algo}, applied to $ \Phi $ with the stochastic quantizer $\mathcal{Q}$. Then, the resulting error satisfies $$\max_{i,j}|\rho(XW)_{ij}-\rho(XQ)_{ij}|\leq 4K\sqrt{2\pi C p \log N_{0}}\max_{1\leq i\leq N_{0}}\|X_{i}\|$$
     and 
     \begin{gather*}
         \max_{i,j}|Q_{i,j}|\leq 2K
     \end{gather*}
     with probability greater than $$ 1-N_{1}\sum_{t=1}^{N_{0}}\sqrt{2}\exp\Big({-\dfrac{C\|X_{t}\|^{2}}{32\pi\max_{1\leq i\leq t-1}\|X_{i}\|^{2}}}\Big)-\sqrt{2}mN_{1}N_{0}^{-p}, $$
	where  $ X_{t} $ represents the $ t $-th column of $ X $.
\end{proposition}
\begin{proof}
	Let $\kappa:=4K\sqrt{2\pi C p \log N_{0}}\max_{1\leq i\leq N}\|X_{i}\|$. Consider any neuron (column) $ w\in\mathbb{R}^{N_{0}} $ from $ W $. Let the quantized neuron $ q $ and accumulated error $ u_{t} $ be as defined in Subsection \ref{tildeu}. By Theorem \ref{lemma:ubound}, we have $u_{t} \prec_{cx} \mathcal{N}(0, \beta_{t}I)$, where $\beta_t$ is defined in Lemma \ref{lemma:matrixbound}. Recall that at each step in Algorithm \ref{algo}, the quantized weight is $q_t=\mathcal{Q}\left(w_t+ \tfrac{\langle u_{t-1},X_t\rangle}{C\|X_t\|^2}\right).$ Let $B_t$ be the event $\{|q_t|> 2K\}$. We know $B_t \subseteq A_t= \left\{|\tfrac{\langle u_{t-1},X_t\rangle}{C\|X_t\|^2}|>K\right\}$ since $|w_t|< K$. To obtain the result in the statement, it suffices to bound the following failure probability
    \begin{gather*}
        \mathbb{P}\left(\left\{\|u_{N_0}\|_\infty > \kappa\right\}\cup A_{1} \cup\dots \cup A_{N_0}\right).
    \end{gather*}

    We now bound the probability of each set individually. First, $\mathbb{P}(A_1)$ is 0 since $u_0=0$. We then control $\mathbb{P}(A_t)$ for $t=2,\dots,N_0$. By Theorem \ref{lemma:ubound} and Lemma \ref{lemma:convex_order}, \ref{linear}, we have
	\begin{gather*}
		\dfrac{\langle u_{t-1},X_{t}\rangle}{C\|X_{t}\|^{2}} \prec_{cx}\mathcal{N}\Big(0,\dfrac{\beta_{t-1}}{C^{2}\|X_{t}\|^{2}}\Big),
	\end{gather*}
	where $ \beta_{t-1}=8C\pi K^{2}\max_{1\leq i\leq t-1}\|X_{i}\|^{2}. $
	Applying Lemma \ref{lemma:convex_order}, \ref{lemma:cx-gaussian-tail}, we obtain
	\begin{gather}\label{ineq:failure}
		\P\Big(\Big|\dfrac{\langle u_{t-1},X_{t}\rangle}{C\|X_{t}\|^{2}}\Big|>K\Big)\leq \sqrt{2}\exp\Big({-\dfrac{C\|X_{t}\|^{2}}{32\pi\max_{1\leq i\leq t-1}\|X_{i}\|^{2}}}\Big).
	\end{gather} 
For $\mathbb{P}\left(\left\{\|u_{N_0}\|_\infty \geq \kappa\right\}\right)$, we use the second part of Lemma \ref{lemma:convex_order}, \ref{lemma:cx-gaussian-tail} with $ \gamma=\sqrt{2}mN_{0}^{-p} $ to get
	\begin{gather}\label{ineq:error}
		\P(\|u_{N_{0}}\|_{\infty}> \kappa)\leq \sqrt{2}mN_{0}^{-p}.
	\end{gather}
    Then we use apply a union bound on the failure probability to obtain an upper bound
    \begin{align*}
         \mathbb{P}\left(\left\{\|u_{N_0}\|_\infty > \kappa\right\}\cup A_{1} \cup\dots \cup A_{N_0}\right)\leq \sqrt{2}mN_{0}^{-p}+\sum_{t=2}^{N_0}\sqrt{2}\exp\Big({-\dfrac{C\|X_{t}\|^{2}}{32\pi\max_{1\leq i\leq t-1}\|X_{i}\|^{2}}}\Big).
    \end{align*}
    This implies that
\begin{equation}\label{quan_bound_one}
\begin{aligned}
	&\P(\|\rho(Xw)-\rho(Xq)\|_{\infty}\leq\kappa \text{ and } \|q\|_\infty \leq 2K)\\
	&\geq 1-\sum_{t=2}^{N_{0}}\sqrt{2}\exp\Big({-\dfrac{C\|X_{t}\|^{2}}{32\pi\max_{1\leq i\leq t-1}\|X_{i}\|^{2}}}\Big)-\sqrt{2}mN_{0}^{-p},
\end{aligned}
\end{equation}
by the definition of $ u_{N_{0}} $ and the the Lipschitz property of the activation function $\rho$. Inequality (\ref{quan_bound_one}) holds true for each column of $ W\in\mathbb{R}^{m\times N_{1}} $. Taking a union bound over all the columns completes the proof.

\end{proof}
\begin{remark}
This proposition ensures that each entry of $|Q|$ is bounded by $2K$ with high probability. Moreover, since $Q$ takes values in the discrete alphabet $\mathcal{A}$ and the only elements of $\mathcal{A}$ within $[-2K, 2K]$ are $-2K$ and $2K$, this implies that each entry of $Q$ must be either $-2K$ or $2K$. Hence, although the quantizer is defined over a larger alphabet, the quantized weights are effectively supported on just two values, realizing 1-bit quantization in practice.
\end{remark}
    %
    %
\begin{remark}
    For multi-layer analysis, we can use the `data alignment' technique introduced in \cite{ZS}, which would allow establishing an error bound for the entire neural network. We leave this to the reader.
\end{remark}
 This result demonstrates that, for one neuron $w$, by choosing $C$ proportional to \(\log(N_0)\), the Euclidean norm error $\|Xw-Xq\|$ associated with quantization scales as \( K\sqrt{m}\cdot {\polylog N_0}\cdot\max_{i}\|X_i\| \) with high probability, where $m$ is the number of data points. In the case when $m > N$, one can modify the argument in the proof of Proposition \ref{prop:onebit_quant} via an SVD of the data matrix $X=U\Sigma V^{\top}$. To be more specific, one can prove that running Algorithm \ref{phaseii} with $X$ is equivalent to running it with $\Sigma V^{\top}$. Moreover, $\|X(w-q)\|=\|\Sigma V^\top (w-q)\|$, thereby yielding an error of order \( K\sqrt{r}\cdot {\polylog N_0}\cdot\max_{i}\|X_i\| \), where $r\leq N_0$ is the rank of $X$. In summary, the error scales like \( K\sqrt{\min\{m,N_0\}}\cdot {\polylog N_0}\cdot\max_{i}\|X_i\| \). This stands in sharp contrast to the Euclidean norm error from the naive round-to-nearest (RTN) algorithm, where each entry of \( w \) is rounded to the nearest element of \( \mathcal{A} \). In that case, the only error bound, in general, is via triangle inequalities and leads to a Euclidean norm error $\|X(w-q)\|\leq  \sum_{i=1}^{N_0}\|X_i\| |w_i-q_i| \leq 2K\cdot  N_0 \cdot \max_i \|X_i\| $. 
\section{Pruning}\label{sparse}
In this section, we specify the choice of \( \mathcal{T} \) used for pruning, and we present the error analysis of applying Algorithm \ref{algo} to a single layer of a neural network.
\subsection{Choice of $ \mathcal{T} $ for Pruning}
Suppose our layer weights are bounded by $ K>0 $ in absolute value. Given $ c>0 $, our goal is to encourage weights below $ cK $ to be set to $0$ while keeping weights above $ cK $ unchanged. To achieve this, we set $ \mathcal{T} $ to the stochastic operator $ \mathcal{S} $, with
\begin{equation}\label{molli}
	\mathcal{S}(z)=\begin{cases}
		z& \text{ if $ |z|> cK$,} \\
		\xi_{z}&\text{ if $ |z|\leq cK $.}
	\end{cases}
\end{equation}
Here $ \xi_{z} $ is a random variable that is set to 0 with probability $ 1-\tfrac{2}{(c+1)K} |z|$ and is sampled from $\sgn(z)\cdot {U}[cK,K] $ with probability $ \tfrac{2}{(c+1)K} |z| $. It is straightforward to verify that the operator $\mathcal{S}$ satisfies the assumptions in Theorem \ref{lemma:ubound} with $M$ being $K$.
\subsection{Error analysis}
Consider again $ \Phi(x)=\rho(W^{\top}x) $, where $ x\in\mathbb{R}^{N_{0}} $ is the input data or activation, $ W\in\mathbb{R}^{N_{0}\times N_{1}} $ is the weight matrix and $ \rho $ is the ReLU (or any 1-Lipschitz) activation function. 
We apply Algorithm \ref{algo} to prune $ \Phi $ using $ X \in\mathbb{R}^{m\times N_{0}}$, where each row represents a data point. We can then deduce the following result.

\begin{proposition}
	Let $C\geq 1$ and fix any $p\geq1$. Suppose that $Q$ is the pruned weight matrix resulting from Algorithm \ref{algo}, applied to $ \Phi $ with the stochastic pruning operator $\mathcal{S}$.
    Then, the resulting error satisfies $$\max_{i,j}|\rho(XW)_{ij}-\rho(XQ)_{ij}| \leq K\sqrt{2C\pi p\log N_{0}}\max_{1\leq i \leq N_{0}}\|X_{i}\| $$ with probability greater than $ 1-\sqrt{2}mN_{1}N_{0}^{-p}  $, where $ X_{t} $ represents the $ t $-th column of $ X $.
\end{proposition}
\begin{proof}
	Consider any neuron (column) $ w\in\mathbb{R}^{\N_{0}} $ from $ W $. Let $ q $ and $ u_{i} $ be defined as in Subsection \ref{tildeu}. We run Algorithm \ref{phaseii} independently for each $w$. 
	
	It is straightforward to verify that the assumptions in Theorem \ref{lemma:ubound} are satisfied with $ M = K $.  Thus, from Theorem \ref{lemma:ubound}, we deduce the accumulated error satisfies
	\begin{gather*}
		u_{t}\prec_{cx} \mathcal{N}(0,\beta_{t}I) \text{ for $ 1\leq t \leq N _{0}$},
	\end{gather*}
	where $\beta_{t}=\tfrac{C\pi K^{2}}{2}\max_{1\leq i\leq t}\|X_{i}\|^{2}.$
	
	Applying Lemma \ref{lemma:convex_order}, \ref{lemma:cx-gaussian-tail} with $ \gamma  = \sqrt{2}mN_{0}^{-p} $, we obtain
	\begin{gather*}
		\P(\|u_{N_{0}}\|_{\infty}> K\sqrt{2C\pi p\log N_{0}}\max_{1\leq i \leq N}\|X_{i}\|)\leq  \sqrt{2}mN_{0}^{-p} .
	\end{gather*}
This implies 
\begin{equation}\label{sparse_one}
\begin{aligned}
	&\P(\|\rho(Xw)-\rho(Xq)\|_{\infty}> K\sqrt{2C\pi p\log N_{0}}\max_{1\leq i \leq N}\|X_{i}\|)
	\leq  \sqrt{2}mN_{0}^{-p} .
\end{aligned}
\end{equation}
	The inequality (\ref{sparse_one}) holds for any column $w$ of $ W\in \mathbb{R}^{N_{0}\times N_{1}}$. Applying a union bound over all the columns completes the proof.
\end{proof}

\section{Quantization with pruning}\label{quant_prune}
In this section, we show that our methods can handle combinations of operators, and provide the choices for $ \mathcal{T} $ used in joint quantization with pruning, followed by an error analysis for applying Algorithm \ref{algo} to a single layer of a neural network.

\subsection{Choice of $ \mathcal{T} $ for Quantization with Pruning}

Consider an $ L $-layer perceptron with all weights strictly bounded by $ K > 0 $ in absolute value.  We again begin with an infinite alphabet $\mathcal{A}= \{\dots,-6K,-4K,-2K,0,2K,4K,6K,\dots\} $. Let $ \mathcal{S} $ be the pruning operator defined in \eqref{molli}, with $ c>0 $, and let  $ \mathcal{Q}:\mathbb{R}\rightarrow \mathcal{A}$ be the stochastic scalar quantizer with \begin{equation*}
		\mathcal{Q}(z)=
		\begin{cases}
                    \left(\lfloor \frac{z}{2K}\rfloor\right)\cdot 2K & \text{with probability $ p_z $}\\
				 \left(\lceil  \frac{z}{2K} \rceil\right)\cdot 2K & \text{with probability $ 1-p_z $},\\
			\end{cases}
	\end{equation*}
    where $ p_z=\lceil  \frac{z}{2K} \rceil - \frac{z}{2K} $. Our goal is to encourage weights below $ cK $ to be set to $0$ and then quantize the weights using $\mathcal{A}$.
%
%
To that end, we define $ \mathcal{T} = \mathcal{Q} \circ \mathcal{S} $.
\subsection{Error analysis} 
With the same notation as before, we can  deduce the following result.

\begin{proposition}
	Let $C\geq 1$, and fix any $p\geq1$. Suppose that $Q$ is the pruned and quantized weight matrix resulting from Algorithm \ref{algo}, applied to $ \Phi $ with the stochastic operator $\mathcal{T}=\mathcal{Q}\circ\mathcal{S}$.
    Then, the resulting error satisfies  $$\max_{i,j}|\rho(XW)_{ij}-\rho(XQ)_{ij}|\leq 2K\sqrt{2\pi C p \log N_{0}}\max_{1\leq i\leq N_0}\|X_{i}\|$$ 
    and
    \begin{gather*}
        \max_{i,j}|Q|_{ij}\leq 2K
    \end{gather*}
    with probability greater than $$ 1-N_{1}\sum_{t=1}^{N_{0}}\sqrt{2}\exp\Big({-\dfrac{C\|X_{t}\|^{2}}{8\pi\max_{1\leq i\leq t-1}\|X_{i}\|^{2}}}\Big)-\sqrt{2}mN_{1}N_{0}^{-p}, $$
	where $ X_{t} $ represents the $ t $-th column of $ X $.
\end{proposition}
\begin{proof}
First, we verify that the operator $ \mathcal{T} $ satisfies the assumptions of Theorem \ref{lemma:ubound} with $ M = 2K $.

    We begin by showing that for any $z \in \mathbb{R}$, the error $|\mathcal{T}(z) - z|$ is bounded by $2K$. Indeed, if $|z| > cK$, then $\mathcal{S}(z) = z$, so $\mathcal{T}(z) = \mathcal{Q}(z)$, which rounds $z$ up or down to an adjacent alphabet element; since the grid size is $2K$, the rounding error is at most $2K$. If $|z| \leq cK$, then $\mathcal{S}(z)$ maps $z$ to a value in $[-K, K]$ with the same sign. Without loss of generality, suppose $z \in [0, K]$, so $\mathcal{S}(z) \in [0, K]$ and $\mathcal{Q}(\mathcal{S}(z)) \in [0, 2K]$. Hence, both $z$ and $\mathcal{T}(z)$ lie in $[0, 2K]$, and the total deviation remains at most $2K$.

For the expectation, we consider two independent probability spaces $(\Omega_1, \mathcal{E}_1, \mathbb{P}_1)$ and $(\Omega_2, \mathcal{E}_2, \mathbb{P}_2)$, which govern the randomness of the operators $\mathcal{Q}_{\omega_1}$ and $\mathcal{S}_{\omega_2}$ respectively, with $\omega_1 \in \Omega_1$ and $\omega_2 \in \Omega_2$. Jointly, the probability space for $\mathcal{T}$ is the direct product of these two spaces. Then, for any $z \in \mathbb{R}$,
\begin{align*} 
\mathbb{E}[\mathcal{Q}_{\omega_1} \circ \mathcal{S}_{\omega_2}(z)] &= \int_{\Omega_2} \int_{\Omega_1} \mathcal{Q}_{\omega_1}(\mathcal{S}_{\omega_2}(z))\, d\mathbb{P}_1(\omega_1) \, d\mathbb{P}_2(\omega_2) \\ &= \int_{\Omega_2} \mathbb{E}_{\omega_1}[\mathcal{Q}_{\omega_1}(\mathcal{S}_{\omega_2}(z))] \, d\mathbb{P}_2(\omega_2) \\ &= \int_{\Omega_2} \mathcal{S}_{\omega_2}(z) \, d\mathbb{P}_2(\omega_2) \\ &= \mathbb{E}_{\omega_2}[\mathcal{S}_{\omega_2}(z)] = z.
\end{align*} The third equality follows from the unbiasedness of $\mathcal{Q}$, and the final equality from the unbiasedness of $\mathcal{S}$. The rest of the proof proceeds identically to that of  Proposition \ref{prop:onebit_quant}, so we omit the details.
\end{proof}


\section*{Acknowledgment}
We gratefully acknowledge partial support by National Science Foundation Grants DMS-2012546 and DMS-2410717.

\bibliographystyle{abbrv}
\bibliography{references}

\begin{thebibliography}{10}

\bibitem{ALS}
R.~Alweiss, Y.~P. Liu, and M.~Sawhney.
\newblock Discrepancy minimization via a self-balancing walk.
\newblock In {\em Proceedings of the 53rd Annual ACM SIGACT Symposium on Theory
  of Computing}, STOC 2021, page 14–20, New York, NY, USA, 2021. Association
  for Computing Machinery.

\bibitem{ACNHH}
S.~Ashkboos, M.~L. Croci, M.~G. do~Nascimento, T.~Hoefler, and J.~Hensman.
\newblock Slice{GPT}: Compress large language models by deleting rows and
  columns.
\newblock In {\em The Twelfth International Conference on Learning
  Representations}, 2024.

\bibitem{CYDGMK}
Y.~Cai, Z.~Yao, Z.~Dong, A.~Gholami, M.~W. Mahoney, and K.~Keutzer.
\newblock Zeroq: A novel zero shot quantization framework.
\newblock In {\em Proceedings of the IEEE/CVF conference on computer vision and
  pattern recognition}, pages 13169--13178, 2020.

\bibitem{CWVCPSG}
J.~Choi, Z.~Wang, S.~Venkataramani, P.~I.-J. Chuang, V.~Srinivasan, and
  K.~Gopalakrishnan.
\newblock Pact: Parameterized clipping activation for quantized neural
  networks.
\newblock {\em arXiv preprint arXiv:1805.06085}, 2018.

\bibitem{CKYK}
Y.~Choukroun, E.~Kravchik, F.~Yang, and P.~Kisilev.
\newblock Low-bit quantization of neural networks for efficient inference.
\newblock In {\em 2019 IEEE/CVF International Conference on Computer Vision
  Workshop (ICCVW)}, pages 3009--3018. IEEE, 2019.

\bibitem{CGFZS}
I.~Colbert, F.~Grob, G.~Franco, J.~Zhang, and R.~Saab.
\newblock Accumulator-aware post-training quantization.
\newblock {\em arXiv preprint arXiv:2409.17092}, 2024.

\bibitem{CBD}
M.~Courbariaux, Y.~Bengio, and J.-P. David.
\newblock Binaryconnect: Training deep neural networks with binary weights
  during propagations.
\newblock {\em Advances in neural information processing systems}, 28, 2015.

\bibitem{EGMCE}
U.~Evci, T.~Gale, J.~Menick, P.~S. Castro, and E.~Elsen.
\newblock Rigging the lottery: Making all tickets winners.
\newblock In {\em International conference on machine learning}, pages
  2943--2952. PMLR, 2020.

\bibitem{FA}
E.~Frantar and D.~Alistarh.
\newblock Sparsegpt: Massive language models can be accurately pruned in
  one-shot.
\newblock In {\em International Conference on Machine Learning}, pages
  10323--10337. PMLR, 2023.

\bibitem{FAHA}
E.~Frantar, S.~Ashkboos, T.~Hoefler, and D.~Alistarh.
\newblock {OPTQ}: Accurate quantization for generative pre-trained
  transformers.
\newblock In {\em The Eleventh International Conference on Learning
  Representations}, 2023.

\bibitem{HKDFY}
Y.~He, G.~Kang, X.~Dong, Y.~Fu, and Y.~Yang.
\newblock Soft filter pruning for accelerating deep convolutional neural
  networks.
\newblock {\em arXiv preprint arXiv:1808.06866}, 2018.

\bibitem{HW}
Z.~Huang and N.~Wang.
\newblock Data-driven sparse structure selection for deep neural networks.
\newblock In {\em Proceedings of the European conference on computer vision
  (ECCV)}, pages 304--320, 2018.

\bibitem{HNHBS}
I.~Hubara, Y.~Nahshan, Y.~Hanani, R.~Banner, and D.~Soudry.
\newblock Improving post training neural quantization: Layer-wise calibration
  and integer programming.
\newblock {\em arXiv preprint arXiv:2006.10518}, 2020.

\bibitem{JKCZTHAK}
B.~Jacob, S.~Kligys, B.~Chen, M.~Zhu, M.~Tang, A.~Howard, H.~Adam, and
  D.~Kalenichenko.
\newblock Quantization and training of neural networks for efficient
  integer-arithmetic-only inference.
\newblock In {\em Proceedings of the IEEE conference on computer vision and
  pattern recognition}, pages 2704--2713, 2018.

\bibitem{LAT}
N.~Lee, T.~Ajanthan, and P.~H. Torr.
\newblock Snip: Single-shot network pruning based on connection sensitivity.
\newblock {\em arXiv preprint arXiv:1810.02340}, 2018.

\bibitem{LZKZXWCYLZ}
L.~Liu, S.~Zhang, Z.~Kuang, A.~Zhou, J.-H. Xue, X.~Wang, Y.~Chen, W.~Yang,
  Q.~Liao, and W.~Zhang.
\newblock Group fisher pruning for practical network compression.
\newblock In {\em International Conference on Machine Learning}, pages
  7021--7032. PMLR, 2021.

\bibitem{LMZGYCS}
Z.~Liu, H.~Mu, X.~Zhang, Z.~Guo, X.~Yang, K.-T. Cheng, and J.~Sun.
\newblock Metapruning: Meta learning for automatic neural network channel
  pruning.
\newblock In {\em Proceedings of the IEEE/CVF international conference on
  computer vision}, pages 3296--3305, 2019.

\bibitem{LS}
E.~Lybrand and R.~Saab.
\newblock A greedy algorithm for quantizing neural networks.
\newblock {\em Journal of Machine Learning Research}, 22(156):1--38, 2021.

\bibitem{MFW}
X.~Ma, G.~Fang, and X.~Wang.
\newblock Llm-pruner: On the structural pruning of large language models.
\newblock {\em Advances in neural information processing systems},
  36:21702--21720, 2023.

\bibitem{MS}
J.~Maly and R.~Saab.
\newblock A simple approach for quantizing neural networks.
\newblock {\em Applied and Computational Harmonic Analysis}, 66:138--150, 2023.

\bibitem{NAVLB}
M.~Nagel, R.~A. Amjad, M.~Van~Baalen, C.~Louizos, and T.~Blankevoort.
\newblock Up or down? adaptive rounding for post-training quantization.
\newblock In {\em International Conference on Machine Learning}, pages
  7197--7206. PMLR, 2020.

\bibitem{SS}
M.~Shaked and J.~Shanthikumar.
\newblock {\em Stochastic Orders}.
\newblock Springer Series in Statistics. Springer New York, 2007.

\bibitem{SCCWGWL}
J.~Su, Y.~Chen, T.~Cai, T.~Wu, R.~Gao, L.~Wang, and J.~D. Lee.
\newblock Sanity-checking pruning methods: Random tickets can win the jackpot.
\newblock {\em Advances in neural information processing systems},
  33:20390--20401, 2020.

\bibitem{TKYG}
H.~Tanaka, D.~Kunin, D.~L. Yamins, and S.~Ganguli.
\newblock Pruning neural networks without any data by iteratively conserving
  synaptic flow.
\newblock {\em Advances in neural information processing systems},
  33:6377--6389, 2020.

\bibitem{WZG}
C.~Wang, G.~Zhang, and R.~Grosse.
\newblock Picking winning tickets before training by preserving gradient flow.
\newblock {\em arXiv preprint arXiv:2002.07376}, 2020.

\bibitem{WLLLH}
K.~Wang, Z.~Liu, Y.~Lin, J.~Lin, and S.~Han.
\newblock Haq: Hardware-aware automated quantization with mixed precision.
\newblock In {\em Proceedings of the IEEE/CVF conference on computer vision and
  pattern recognition}, pages 8612--8620, 2019.

\bibitem{WCHC}
P.~Wang, Q.~Chen, X.~He, and J.~Cheng.
\newblock Towards accurate post-training network quantization via bit-split and
  stitching.
\newblock In {\em International Conference on Machine Learning}, pages
  9847--9856. PMLR, 2020.

\bibitem{ZYYH}
D.~Zhang, J.~Yang, D.~Ye, and G.~Hua.
\newblock Lq-nets: Learned quantization for highly accurate and compact deep
  neural networks.
\newblock In {\em Proceedings of the European conference on computer vision
  (ECCV)}, pages 365--382, 2018.

\bibitem{ZS}
J.~Zhang and R.~Saab.
\newblock Spfq: A stochastic algorithm and its error analysis for neural
  network quantization.
\newblock {\em arXiv preprint arXiv:2309.10975}, 2023.

\bibitem{ZZS}
J.~Zhang, Y.~Zhou, and R.~Saab.
\newblock Post-training quantization for neural networks with provable
  guarantees.
\newblock {\em SIAM Journal on Mathematics of Data Science}, 5(2):373--399,
  2023.

\bibitem{ZNZZZT}
C.~Zhao, B.~Ni, J.~Zhang, Q.~Zhao, W.~Zhang, and Q.~Tian.
\newblock Variational convolutional neural network pruning.
\newblock In {\em 2019 IEEE/CVF Conference on Computer Vision and Pattern
  Recognition (CVPR)}, pages 2775--2784, 2019.

\bibitem{ZHDDZ}
R.~Zhao, Y.~Hu, J.~Dotzel, C.~De~Sa, and Z.~Zhang.
\newblock Improving neural network quantization without retraining using
  outlier channel splitting.
\newblock In {\em International conference on machine learning}, pages
  7543--7552. PMLR, 2019.

\bibitem{ZYGXC}
A.~Zhou, A.~Yao, Y.~Guo, L.~Xu, and Y.~Chen.
\newblock Incremental network quantization: Towards lossless cnns with
  low-precision weights.
\newblock {\em arXiv preprint arXiv:1702.03044}, 2017.

\end{thebibliography}
\end{document}